%% file: main.tex
\documentclass[twoside]{article}

\usepackage[accepted]{aistats/aistats2025}
\usepackage{url}
\input{preamble}

\begin{document}

%
\runningtitle{Relating Piecewise Linear KANs to ReLU Networks}

%

\twocolumn[

\aistatstitle{Relating Piecewise Linear Kolmogorov Arnold Networks \\to ReLU Networks}

\aistatsauthor{ Nandi Schoots \And Mattia Jacopo Villani \And Niels uit de Bos }

\aistatsaddress{ University of Oxford \And  King's College London \And MATS } ]

\input{content}


\renewcommand{\refname}{}
\subsubsection*{References}

\bibliography{references}

\section*{Checklist}

 \begin{enumerate}

 \item For all models and algorithms presented, check if you include:
 \begin{enumerate}
   \item A clear description of the mathematical setting, assumptions, algorithm, and/or model. [Yes]
   \item An analysis of the properties and complexity (time, space, sample size) of any algorithm. [Not Applicable]
   \item (Optional) Anonymized source code, with specification of all dependencies, including external libraries. [Not Applicable]
 \end{enumerate}

 \item For any theoretical claim, check if you include:
 \begin{enumerate}
   \item Statements of the full set of assumptions of all theoretical results. [Yes]
   \item Complete proofs of all theoretical results. [Yes]
   \item Clear explanations of any assumptions. [Yes]     
 \end{enumerate}

 \item For all figures and tables that present empirical results, check if you include:
 \begin{enumerate}
   \item The code, data, and instructions needed to reproduce the main experimental results (either in the supplemental material or as a URL). [Yes/No/Not Applicable]
   \item All the training details (e.g., data splits, hyperparameters, how they were chosen). [Not Applicable]
         \item A clear definition of the specific measure or statistics and error bars (e.g., with respect to the random seed after running experiments multiple times). [Not Applicable]
         \item A description of the computing infrastructure used. (e.g., type of GPUs, internal cluster, or cloud provider). [Not Applicable]
 \end{enumerate}

 \item If you are using existing assets (e.g., code, data, models) or curating/releasing new assets, check if you include:
 \begin{enumerate}
   \item Citations of the creator If your work uses existing assets. [Not Applicable]
   \item The license information of the assets, if applicable. [Not Applicable]
   \item New assets either in the supplemental material or as a URL, if applicable. [Not Applicable]
   \item Information about consent from data providers/curators. [Not Applicable]
   \item Discussion of sensible content if applicable, e.g., personally identifiable information or offensive content. [Not Applicable]
 \end{enumerate}

 \item If you used crowdsourcing or conducted research with human subjects, check if you include:
 \begin{enumerate}
   \item The full text of instructions given to participants and screenshots. [Not Applicable]
   \item Descriptions of potential participant risks, with links to Institutional Review Board (IRB) approvals if applicable. [Not Applicable]
   \item The estimated hourly wage paid to participants and the total amount spent on participant compensation. [Not Applicable]
 \end{enumerate}

 \end{enumerate}

\newpage 
\appendix

\input{appendix}

\end{document}

%% file: preamble.tex
\usepackage{amsthm} 
\usepackage{amsmath}
\usepackage{amssymb} 
\usepackage{thmtools} 
\usepackage{mathtools} 
\usepackage{stmaryrd}
\usepackage{cuted}
\usepackage{quiver}

\usepackage{xspace}
\usepackage{float}
\usepackage{enumitem}
\usepackage[framemethod=TikZ]{mdframed}
\usepackage{xcolor}

\usepackage{algorithm}
\usepackage{algorithmic}

\newtheorem{definition}{Definition}
\newtheorem{theorem}{Theorem}

\newtheorem{lemma}{Lemma}
\newtheorem{proposition}{Proposition}

\usepackage{natbib}

\usepackage[capitalise,noabbrev]{cleveref}

\newcommand{\defn}[1]{\emph{#1}}
\newcommand{\dimnin}{\ensuremath{{n_{\mathrm{in}}}}}
\newcommand{\dimnout}{\ensuremath{{n_{\mathrm{out}}}}}

\newcommand{\kanlayer}{\ensuremath{\mathbf{\Phi}}}
\newcommand\mydots{\makebox[1em][c]{.\hfil.\hfil.}}
\DeclareMathOperator{\relu}{ReLU}

\newcommand{\reluspace}{\textbf{ReLU}}
\newcommand{\kanspace}{\textbf{KAN}}

\newcount\foreachcount

\makeatletter
\let\ea\expandafter
\def\foreachLetter#1#2#3{\foreachcount=#1
  \ea\loop\ea\ea\ea#3\@Alph\foreachcount
  \advance\foreachcount by 1
  \ifnum\foreachcount<#2\repeat}
\def\definecal#1{\ea\gdef\csname c#1\endcsname{\ensuremath{\mathcal{#1}}\xspace}}
\foreachLetter{1}{27}{\definecal}

\makeatletter
\let\ea\expandafter
\def\foreachLetter#1#2#3{\foreachcount=#1
  \ea\loop\ea\ea\ea#3\@Alph\foreachcount
  \advance\foreachcount by 1
  \ifnum\foreachcount<#2\repeat}
\def\definecal#1{\ea\gdef\csname b#1\endcsname{\ensuremath{\mathbf{#1}}\xspace}}
\foreachLetter{1}{27}{\definecal}

\renewcommand{\bR}{\mathbb{R}}

%% file: content.tex
\begin{abstract}
Kolmogorov-Arnold Networks are a new family of neural network architectures which holds promise for overcoming the curse of dimensionality and has interpretability benefits \citep{liu2024kan}. 
In this paper, we explore the connection between Kolmogorov Arnold Networks (KANs) with piecewise linear (univariate real) functions and ReLU networks.
We provide completely explicit constructions to convert a piecewise linear KAN into a ReLU network and vice versa.

\end{abstract}

\section{INTRODUCTION}

Architectural innovations are key drivers in the evolution of deep learning. 
Advances in architecture design, such as the introduction of convolutional \citep{lecun1989conv} or attention layers~\citep{vaswani2017attention}, yield significant performance improvements in AI systems.
Very recently, \cite{liu2024kan} introduced Kolmogorov Arnold Networks (KANs), 
an alternative to feedforward-style architectures in deep learning.
The authors argue that KANs are more interpretable than traditional feedforward networks.
However, they also find that KANs are typically 10x slower to train than MLPs, 
given the same number of parameters.

Our paper introduces completely explicit constructions for converting a ReLU network into a KAN with piecewise linear activation functions and vice versa~(\cref{sec:converting-pl-kans-to-relu}).
This means users can train a ReLU network, translate it to a KAN, and benefit from the enhanced interpretability of the KAN.
Moreover, unifying the architectures under a common framework facilitates the application of existing tools and theories developed for the ReLU network, such as analyzing symmetries and polyhedral regions, or initialisation techniques and research on generalisation bounds. 
In other words, we can have the best of both worlds.

Our conversion process is efficient in terms of the number of non-zero parameters
of the converted network: the KAN-to-ReLU conversion does not increase the number of non-zero parameters,
while the ReLU-to-KAN conversion increases the number of non-zero parameters
by a term that is linear in the number of neurons (\cref{sec:conversion-efficiency}).
However, in the KAN-to-ReLU conversion,
we end up with a very wide network with sparse weight matrices
(\cref{sec:class-embedding}).


We show that, for a given parameter budget, KANs produce a finer polyhedral complex than ReLU networks. 
Specifically, we show that the upper bound on the number of linear regions implemented by a KAN is higher 
(\cref{sec:polyhedral-decomposition}).  
Parameter efficiency is key to enabling the use of lightweight models at inference time.


Throughout this paper, the term KAN refers to a KAN
with piecewise linear activation functions.

\section{RELATED WORKS}

Kolmogorov Arnold's result, also known as the \textit{Kolmogorov Superposition Theorem} (KST) 
shows that every function can be written using univariate functions and summing \citep{kolmogorov1956superpositions}.
The recent \cite{liu2024kan} construction relies on this result.

Previously, several other attempts to unify KST and Deep Learning theory have been made \citep{schmidt2021kolmogorov, ismayilova2024kolmogorov}.

KANs represent multivariate functions as compositions and superpositions of univariate functions. 
These representations are often considered more interpretable \citep{yang2021gami} because they are based on univariate functions.
However, these functions can be very complex, and, e.g., in the case of piecewise linear univariate functions, they may have a large number of pieces (on which the function does not necessarily monotonically increase). 
Additionally, constructive proofs of the Kolmogorov Arnold theorem typically find highly irregular and erratic univariate inner and outer functions \citep{braun2009constructive}, decreasing their interpretability. 
Some authors try to impose Lipschitz continuity to enforce higher regularity in the inner and outer functions \citep{actor2017algorithm}; however, this comes at a cost of a large number of total functions. 
Moreover, the large number of univariates increases network complexity. 

\section{BACKGROUND}
\label{sec:background}

In this section, we recall some of the core ideas and definitions from \cite{liu2024kan} for the reader's benefit.
We also discuss piecewise linear functions and ReLU networks.

The \textbf{Kolmogorov Arnold Theorem}
(or the superposition theorem) states the following. 
Let $f \colon [0,1]^n \rightarrow \mathbb{R}$ be a continuous multivariate real function.
Then there are a finite number of continuous univariate real functions $\Phi_q$ and $\phi_{p}^q$ such that $f$ can be written as
\[f(\mathbf{x}) = \sum_{q=1}^{2n+1} \Phi_q \left( \sum_{p=1}^n \phi_{p}^q(x_p) \right).\]


\textbf{Kolmogorov Arnold Networks} (KANs)
were recently introduced and generalize Kolmogorov Arnold representations.
\begin{definition}
A \defn{KAN layer} with input dimension $\dimnin$ and output dimension $\dimnout$ is given by a $\dimnout$-by-$\dimnin$ matrix
$\kanlayer = \{\phi_{p}^q\}_{p=1, \mathellipsis, \dimnin }^{q=1, \mathellipsis, \dimnout}$ of univariate real functions $\phi_{p}^q \colon \bR \to \bR$ that we call \defn{activation functions}.
It represents the function
\begin{equation*}
\begin{split}
    \kanlayer \colon \bR^{\dimnin} & \to \bR^\dimnout, \\
     \mathbf x = (x_i)_{i=1, \ldots, \dimnin} & \mapsto \kanlayer(\mathbf x) = \left( \sum_{i=1}^{\dimnin} \phi_{i}^j(x_i) \right)_{j=1, \ldots, \dimnout}
\end{split}
\end{equation*}
\end{definition}

\begin{definition}
A \defn{Kolmogorov Arnold network (KAN)} is a composition of $L$ KAN layers $\kanlayer_{L-1} \circ \mathellipsis \circ \kanlayer_0$.
In the case that the last layer has output dimension 1, the function represented by the KAN takes the form

\begin{align}\label{eq:kan} 
&f(\mathbf{x}) = \mspace{-5mu}
       \sum_{i_{L-1}=1}^{n_{L-1}} \phi_{i_{L-1}}^{L-1, i_L} \mspace{-5mu}
       \left( \mspace{-3mu} \mydots \sum_{i_1=1}^{n_1} \phi_{i_1}^{1, i_2} 
       \left( \sum_{i_0=1}^{n_0} \phi_{i_0}^{0, i_1} (x_{i_0}) \right) \mydots \mspace{-3mu} \right)
\end{align}
where $n_\ell$ is the input dimension of the $\ell$-th KAN layer
$\kanlayer_\ell = \{\phi_{p}^{\ell, q}\}_{p=1, \mathellipsis, n_{\ell}}^{q=1, \mathellipsis, n_{\ell+1}}$.
If a KAN has $L$ layers, we also sometimes say that it has $L - 1$ \defn{hidden layers}.
\end{definition} 


A \textbf{piecewise linear KAN} is a KAN in which each activation function $\phi_{i_{l-1}}^{l-1,i_l}:\mathbb{R}\rightarrow \mathbb{R}$ is piecewise linear
with a finite number of segments. 
Going forward, in the interest of presentation, whenever we say KAN we typically mean a piecewise linear KAN,
but sometimes we will emphasize the piecewise linearity explicitly.

Any piecewise linear function $f$ can be represented as a \textbf{polyhedral complex} $\mathcal{C}(f) = (\Omega, (\alpha_\omega, \beta_\omega)_{\omega \in \Omega})$,
where $\Omega$ is a partition of the input space $\mathbb{R}^{n}$, 
and $(\alpha_\omega, \beta_\omega) \in \bR \times \bR$
are linear coefficients for $f|_\omega$, i.e., $f|_\omega(x) = \alpha_\omega x + \beta_\omega$.

\textbf{Re}ctified \textbf{L}inear \textbf{U}nit (ReLU) networks are a popular family of architectures for deep learning. 
Both ReLU networks and piecewise linear KANs are examples of piecewise linear functions. 
This means that both can be represented as polyhedral complexes through a polyhedral decomposition. 
In the ReLU case, such decompositions have received large theoretical \citep{montufar2014number, arora2016understanding, serra2018bounding} and empirical \citep{raghu2017expressive, humayun2022exact, berzins2023polyhedral, masden2022algorithmic} attention, and their properties are an object of interest. 
In this paper we develop the first analysis of the polyhedral decomposition of piecewise linear KANs.

\section{CONVERTING KANS TO RELU NETWORKS AND VICE VERSA}
\label{sec:converting-pl-kans-to-relu}

In this section we provide methods for translating a ReLU to a KAN with piecewise linear activation functions and vice versa.
See \cref{appendix:b-spline-kans} for a discussion of how the ideas in this section
can be extended to convert KANs with B-spline activation functions
to and from a more unconventional feedforward architecture with both ReLU activation functions and monomial activation functions.

\subsection{For every ReLU there's a KAN}

\begin{theorem}\label{thm:relu-to-kan}
    Let $g \colon \bR^n \rightarrow \bR^m$ be a feedforward network with activation functions from a family $\mathcal F$.
    There exists a KAN $f\colon \bR^n \rightarrow \bR^m$
    with activation functions that are either affine linear or from $\mathcal F$
    such that $f(x) = g(x)$ for all $ x \in \mathbb{R}^n$.
    In particular, if $g$ is a ReLU network, then there exists a piecewise linear KAN $f$ with
    $f(x) = g(x)$ for all $x \in \bR^n$.
\end{theorem}

\begin{proof}
Suppose that $g$ is a one layer network with weight vector $W$ and bias $b$, 
then we can write a KAN as follows
\[f(x) = \sigma \left( \sum_{i_0 = 1}^n \phi_{i_0}(x_{i_0}) \right),\]
where $\sigma$ is the univariate real (ReLU) activation function, 
$\phi_{i_0}(x_{i_0}) = w_{i_0} x_{i_0} + b$ for $i_0 = 1$ and
$\phi_{i_0}(x_{i_0}) = w_{i_0} x_{i_0} $ for $i_0 > 1$.

We will now give a general formulation of the KAN corresponding to a network with $L$ layers.
Let $g \colon \mathbb{R}^n \rightarrow \mathbb{R}^m$ be a composition of $L \in \mathbb{N}$ layers, where for each $l \in \{0, \ldots, L-2\}$ the layer is given by: 
\[ \chi^{(l+1)} = \sigma( {W^{(l)}} \chi^{(l)} + B^{(l)}),\]
where $\sigma$ is an element-wise activation function and $\chi^{(l)}$ are the activations,
with $\chi^{(0)} = x$ and output layer  
$g(x) = W^{(L-1)} \chi^{(L-1)} + B^{(L-1)}.$

We define 
\begin{align*} 
&f(\mathbf{x}) = \mspace{-5mu}
       \sum_{i_{L-1}=1}^{n_{L-1}} \phi_{i_{L-1}}^{L-1, i_L} \mspace{-5mu}
       \left( \mspace{-3mu} \mydots \sum_{i_1=1}^{n_1} \phi_{i_1}^{1, i_2} 
       \left( \sum_{i_0=1}^{n_0} \phi_{i_0}^{0, i_1} (x_{i_0}) \right) \mydots \mspace{-3mu} \right)
\end{align*}
where for $l=0$ we define 
\begin{align*}
    \phi_{i_0}^{0, i_1}(x_{i_0}) &= W^{(0)}_{i_0, i_1} x_{i_0} + B^{(0)}_{i_1} \text{ for }i_0 = 1 \text{ and } \\
    \phi_{i_0}^{0, i_1}(x_{i_0}) &= W^{(0)}_{i_0, i_1} x_{i_0} \hspace{1cm} \text{ for }i_0 > 1;
\end{align*}
and for $l>0$ we define 
\begin{align*}
    \phi_{i_{l}}^{l, i_{l+1}}(s) &= W^{(l)}_{i_{l}, i_{l+1}} \sigma (s) + B^{(l)}_{i_{l+1}} \text{ for }i_{l} = 1 \text{ and } \\
    \phi_{i_{l}}^{l, i_{l+1}}(s) &= W^{(l)}_{i_{l}, i_{l+1}} \sigma (s) \hspace{0.9cm}  \text{ for }i_{l} > 1.
\end{align*}
This function $f$ is a KAN and by construction $f(x) = g(x)$ for all $ x \in \mathbb{R}^n$.
\end{proof}

\cite{arora2016understanding} proves that every piecewise linear function with finitely many pieces can be represented exactly by a ReLU network of finite depth and width.
A corollary of Theorem \ref{thm:relu-to-kan} is that any piecewise linear function with finitely many pieces can also be represented exactly as a piecewise linear KAN.
Moreover, it is possible to do so with a KAN network of the same depth as the ReLU network.
Upper bounds on this depth are given by \cite{arora2016understanding}. 

\subsection{For Every KAN there's a ReLU}

\begin{figure}
\centering
    \begin{center}
    \includegraphics[width=0.33\textwidth]{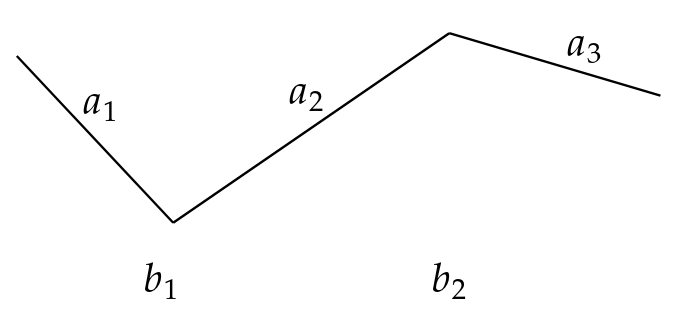}
    \end{center}
  \caption{Example of a piecewise linear activation function.}
  \label{fig:breakpoints}
\hspace{1cm}
    \begin{center}
    \includegraphics[width=0.4\textwidth]{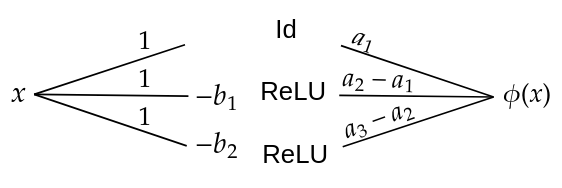}
    \end{center}
  \caption{Network implementing the activation function. \vspace{0.37cm}}
  \label{fig:small-network}
\end{figure}

In this section we show that we can express any KAN as a ReLU network.
We begin with an example of a piecewise linear activation function $\phi$ with two breakpoints $b_1, b_2$ and three different inclinations $a_1, a_2, a_3$, as in Figure \ref{fig:breakpoints}.
If we assume the first segment passes through the origin, then we can write this activation function as
\begin{align*}
   g(x) := a_1 \cdot x & + (a_2 -a_1) \cdot \relu(x - b_1) \\
                       & + (a_3 -a_2) \cdot \relu(x-b_2). 
\end{align*}
We will now check that $\phi(x) = g(x)$ in all three segments:
\begin{itemize}
    \item For $x < b_1$ we get $g(x) = a_1 \cdot x$.
    \item For $b_1 < x < b_2$ we get $g(x) = a_1 \cdot x + (a_2 -a_1) \cdot (x - b_1)$, the derivative of $g$ here is $a_2$ and the value of $g$ at point $b_1$ is $g(b_1) = a_1 \cdot b_1 = \phi(b_1)$.
    \item For $x > b_2$ we get $g(x) = a_1 \cdot x + (a_2 -a_1) \cdot (x - b_1) + (a_3 -a_2) \cdot(x-b_2)$,
    the derivative here is $a_3 $ and the value of $g$ at point $b_2$ is
    $g(b_2) = a_1 \cdot b_2 + (a_2 -a_1) \cdot (b_2 - b_1) = \phi(b_2)$.
\end{itemize}

We will now prove a lemma about a single piecewise linear activation function, a lemma about a single KAN layer, and finally the general theorem about KANs.

\begin{lemma}
    \label{lemma:activation-function-to-relu}
    Let $\phi \colon \bR \to \bR$ be a piecewise linear function with a finite number $n$ of segments.
    Then there exist a $n$-by-1 matrix $W^{(1)}$, a 1-by-$n$ matrix $W^{(2)}$, a bias vector $B^{(1)} \in \bR^n$, and a bias $B^{(2)} \in b$
    such that for all $x \in \bR$
    \begin{equation*}
        \phi(x) = W^{(2)} \relu(W^{(1)} x + B^{(1)}) + B^{(2)};
    \end{equation*}
    in other words, we can write $\phi$ as a feedforward network with one hidden layer with $n$ neurons.
\end{lemma}
\begin{proof}
    Let $b_1, \mathellipsis, b_{n-1} \in \bR$ be the breakpoints of the piecewise linear map $\phi$,
    and let $a_i$ for $1 < i < n$ denote the slope of $\phi$ on the interval $[b_{i-1}, b_{i}) \subset \bR$;
    by $a_1$ we denote the slope on $(-\infty, b_1)$, and by $a_{n}$ we denote the slope on $[b_{n-1}, \infty)$.
    Let $c \in \bR$ be the $y$-intercept of the first segment, i.e., $\phi(x) = a_0 x + c$ for $x \in (-\infty, b_1)$.

    Then we can take
    \begin{align*}        
        W^{(1)} &= (a_1, a_2 - a_1, a_3 - a_2, \mathellipsis, a_{n} - a_{n-1})^T \\
        W^{(2)} &= (1, 1, \mathellipsis, 1) \\
        B^{(1)} &= (0, -b_1, -b_2, \mathellipsis, -b_{n-1}) \\
        B^{(2)} &= c
    \end{align*}
    A simple calculation shows that this is correct.
    Indeed, $W^{(1)} x + B^{(1)}$ is equal to
    $(x, x - x_1, x - x_2, \mathellipsis, x - x_{n-1})$,
    so for $1 \leq i < n - 1$ and $x \in [x_i, x_{i+1})$,
    we see that $\relu(W^{(1)} x + B^{(1)})$ is equal to
    $(x, x - x_1, \mathellipsis, x - x_i, 0, \mathellipsis, 0)$.
    Multiplying this by $W^{(2)}$, we get
    \begin{equation*}
        a_1 x + (a_2 - a_1) (x - b_1) + \mathellipsis + (a_{i+1} - a_{i}) (x - b_i)
    \end{equation*}
    which has a linear coefficient for $x$ equal to $a_{i+1}$, as expected.
    Similarly, we can show that the slopes on the first and last segment are correct.
    The coefficient $B^{(2)}$ then ensures the right intercept on the first segment.
    It follows that the other intercepts are also correct,
    because the function is continuous and has the correct derivative everywhere.
\end{proof}

This demonstrates that we can convert a single activation function in a piecewise linear KAN to a ReLU network.
We use this to convert a single piecewise linear KAN layer.

\begin{lemma}
\label{lem:single-kanlyaer-to-relu-network}
    Let $\kanlayer = \{\phi_{p}^q\}_{p=1, \mathellipsis, \dimnin}^{q = 1, \mathellipsis, \dimnout} \colon \bR^\dimnin \to \bR^\dimnout$
    be a KAN layer with piecewise linear activation functions $\phi_{q,p}$ with finite numbers of segments.
    Then there exist matrices $W^{(1)}, W^{(2)}$ and bias vectors $B^{(1)}, B^{(2)}$
    such that for all $x \in \bR^\dimnin$
    \begin{equation*}
        \kanlayer(x) = W^{(2)} \relu(W^{(1)} x + B^{(1)}) + B^{(2)};
    \end{equation*}
    in other words, we can write $\kanlayer$ as a feedforward network with one hidden layer.
\end{lemma}
\begin{proof}    
    First assume $\dimnout = 1$;
    the function we need to represent is then $\kanlayer(x) = \sum_{p=1}^\dimnin \phi_{p}^1(x_p)$.
    By \cref{lemma:activation-function-to-relu},
    we can write each $\phi_{p}^1$ as
    \begin{equation*}
        \phi_{p}^1(x_p) = W^{(2)}_{ p} \relu( W^{(1)}_{ p} x_p + B^{(1)}_{ p} ) + B^{(2)}_{ p}.
    \end{equation*}
    We can combine these results together.
    This process is illustrated in \cref{fig:concatenated-vectors} 
    and the equation below outlines the matrix calculus:
    \begin{equation*}
    \begin{split}
    \kanlayer(x) 
    =& \sum_{p=1}^\dimnin \phi_{1,p}(x_p) \\
     =& \sum_{p=1}^\dimnin W^{(2)}_{ p} \relu( W^{(1)}_{ p} x_p + B^{(1)}_{ p} ) + B^{(2)}_{ p} \\
    =& \sum_{p=1}^{\dimnin} B^{(2)}_{ p} + \begin{pmatrix} 
            W^{(2)}_{1} & W^{(2)}_{2} & \cdots & W^{(2)}_{ \dimnin}        
            \end{pmatrix} \cdot \relu\\
            & \left( \mspace{-5mu}
            \begin{pmatrix}
            W^{(1)}_{1} & 0 & \mydots & 0 \\
            0 & W^{(1)}_{2} & \mydots & 0 \\
            \vdots & \vdots & \ddots & \vdots \\
            0 & 0 & \mydots & W^{(1)}_{ \dimnin}
            \end{pmatrix}
            \mspace{-5mu}
            \begin{pmatrix}
            x_1 \\
            x_2 \\
            \vdots \\
            x_\dimnin            
            \end{pmatrix}
            \mspace{-5mu}
            +
            \mspace{-5mu}
            \begin{pmatrix}
            B^{(1)}_{1}  \\
            B^{(1)}_{2}  \\
            \vdots \\
            B^{(1)}_{ \dimnin}
            \end{pmatrix}     
            \mspace{-5mu}
            \right)        
    \end{split}
    \end{equation*}
    This shows that we can also write 
    \[ \kanlayer(x) = W^{(2)} \relu(W^{(1)} x + B^{(1)}) + B^{(2)}.\]

    Now consider again the general case of $\dimnout \geq 1$.
    By what we have just proven, for every $q = 1, \mathellipsis, \dimnout$,
    we can write \begin{equation*}
        \sum_{p=1}^\dimnin \phi_{p}^q(x_p) = W^{(2)}_{ q} \relu(W^{(1)}_{q} x + B^{(1)}_{ q})
    \end{equation*}
    for some matrices $W^{(1)}_{q}, W^{(2)}_{q}, B^{(1)}_{q}, B^{(2)}_{q}$.
    We can again do all these computations in parallel,
    as illustrated on the left-hand side of \cref{fig:big-network};
    more rigorously, we have the following block matrix computation:
    \begin{equation*}
    \begin{split}
        \kanlayer(x) 
        &=  \begin{pmatrix}
            \sum_{p=1}^{\dimnin} \phi_{1,p}(x_p) \\
            \sum_{p=1}^{\dimnin} \phi_{2,p}(x_p) \\
            \vdots \\
            \sum_{p=1}^{\dimnin} \phi_{\dimnout,p}(x_p)          
        \end{pmatrix} \\
        & = \begin{pmatrix}
            W^{(2)}_{ 1} \relu (W^{(1)}_{1} x + B^{(1)}_{1}) + B^{(2)}_{1} \\
            W^{(2)}_{ 2} \relu (W^{(1)}_{2} x + B^{(1)}_{2}) + B^{(2)}_{2} \\
            \vdots \\
            W^{(2)}_{ \dimnout} \relu (W^{(1)}_{\dimnout} x + B^{(1)}_{\dimnout}) + B^{(2)}_{\dimnout}
        \end{pmatrix} \\
        &= \begin{pmatrix} 
            W^{(2)}_{1} & 0 & \cdots & 0 \\
            0 & W^{(2)}_{2} & \cdots & 0 \\
            \vdots & \vdots & \ddots & \vdots \\
            0 & 0 & \cdots & W^{(2)}_{ \dimnout}
        \end{pmatrix} \cdot \relu\\
        & \quad \left(
            \begin{pmatrix}
            W^{(1)}_{1} \\
            W^{(1)}_{2} \\
            \vdots \\
            W^{(1)}_{ \dimnout}
            \end{pmatrix}
            x
            +
            \begin{pmatrix}
            B^{(1)}_{1}  \\
            B^{(1)}_{2}  \\
            \vdots \\
            B^{(1)}_{ \dimnout}
            \end{pmatrix}
        \right)
        + 
        \begin{pmatrix} 
            B^{(2)}_{1}  \\
            B^{(2)}_{2}  \\
            \vdots \\
            B^{(2)}_{ \dimnout}
        \end{pmatrix}                
    \end{split}
    \end{equation*} 
    so again we see that we can write 
    \[ \kanlayer(x) = W^{(2)} \relu(W^{(1)} x + B^{(1)}) + B^{(2)} . \]

\end{proof}

    \begin{figure}
    \begin{center}
    \includegraphics[width=0.4\textwidth]{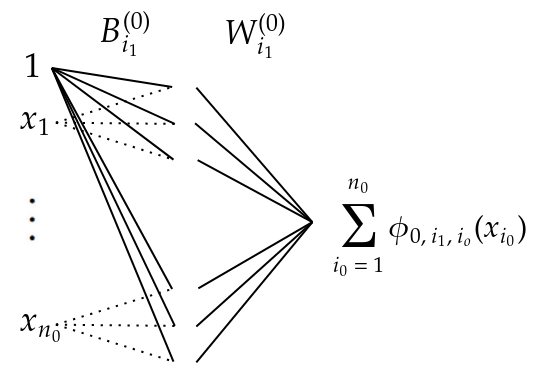}
    \end{center}
    \caption{Concatenating vectors $W^{(1)}_{i_1, i_{0}}$ and $B^{(1)}_{i_1, i_{0}}$ into vectors $W^{(1)}_{i_1}$ and $B^{(1)}_{i_1}$.}
    \label{fig:concatenated-vectors}
    \end{figure}

\begin{theorem}\label{thm:kan-to-relu}
    For every piecewise linear KAN $f:\mathbb{R}^n\rightarrow \mathbb{R}$ there exists a ReLU network $g:\mathbb{R}^n \rightarrow \mathbb{R}$ such that $f(x) = g(x)$ for all $ x \in \mathbb{R}^n$.
\end{theorem}

\begin{proof}
    A KAN with $L$ layers is a composition $\kanlayer_{L-1} \circ \kanlayer_{L-2} \circ \mathellipsis \circ \kanlayer_0$ of $L$ KAN layers.
    By \cref{lem:single-kanlyaer-to-relu-network}, each KAN layer $\kanlayer_\ell$ can be written as
    \begin{equation*}
        \kanlayer_\ell(x^{(\ell)}) = W^{(\ell, 2)} \relu (W^{(\ell, 1)} x^{(\ell)} + B^{(\ell, 1)} ) + B^{(\ell, 1)}.
    \end{equation*}
    When we compose two layers $\kanlayer_{\ell + 1} \circ \kanlayer_\ell$,
    the last layer of the feedforward architecture of $\kanlayer_\ell$ is not followed by a non-linear activation function,
    so it can be combined with the first layer of $\kanlayer_{\ell + 1}$;
    this combination looks like this:
    \begin{equation*}
    \begin{split}        
        \phantom{=}& W^{(\ell + 1, 1)} \kanlayer_\ell(x^{(\ell)}) + B^{(\ell + 1, 1)} \\
        =& W^{(\ell +1 , 1)}
           \left(
                W^{(\ell, 2)} \relu (W^{(\ell, 1)} x^{(\ell)} + B^{(\ell, 1)} ) + B^{(\ell, 2)}
           \right) \\
         & + B^{(\ell + 1, 1)} \\
        =& \underbrace{W^{(\ell +1 , 1)} W^{(\ell, 2)}}_{\text{combined weights}} \relu 
            \left(
                W^{(\ell, 1)} x^{(\ell)}
            +  B^{(\ell, 1)} \right) \\
                &+ \underbrace{W^{(\ell + 1, 1)} B^{(\ell, 2)}
                + B^{(\ell + 1, 1)}}_{\text{combined bias}}.
    \end{split}
    \end{equation*}
\end{proof}

    \begin{figure}
    \begin{center}
    \includegraphics[width=0.5\textwidth]{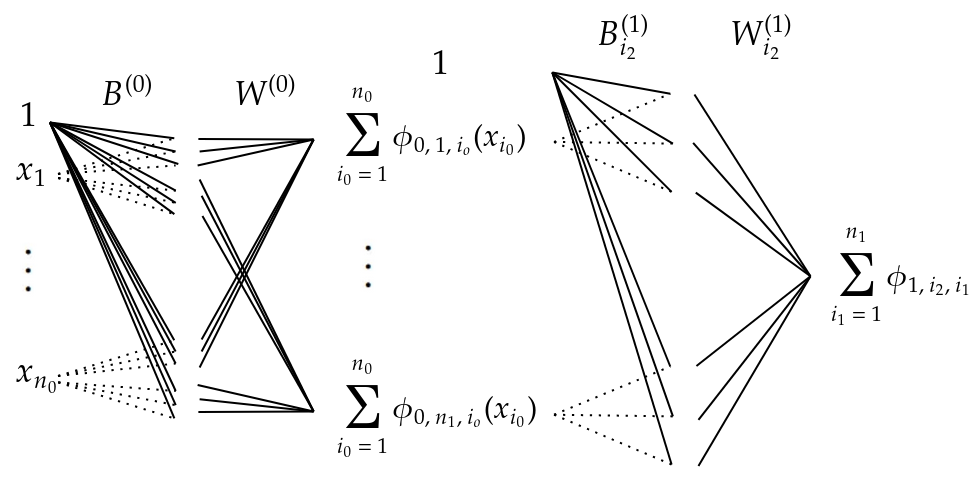}
    \end{center}
    \caption{Three hidden layer network implementing a KAN of depth two.}
    \label{fig:big-network}
    \end{figure}

\subsection{Class Embeddings}
\label{sec:class-embedding}

Let $\textbf{KAN}(L,n,k)$ be the functional class of KANs with $L$ layers, width $n = \text{max}_{i = 1,...,L}(n_i)$, and activation functions with at most $k+1$ segments.
Similarly, let $\textbf{ReLU}(L,n)$ denote the class of ReLU networks with width $n$ and $L$ layers. 



\begin{theorem}
Using the notation from above,
the constructions in
\cref{thm:kan-to-relu}
and
\cref{thm:relu-to-kan}
define the following embeddings:
\begin{equation*}
    \begin{split}
    \textbf{KAN}(L,n,k) 
    &\subseteq \textbf{ReLU}(L+1,n^2 (k+1)) \\
    &\subseteq \textbf{KAN}(L+1,n^2 (k+1),1).
    \end{split}
\end{equation*}
    
\end{theorem}

\begin{proof}

Our KAN-to-ReLU conversion in \cref{thm:kan-to-relu}
converts a KAN with $L$ layers into a ReLU network with $L + 1$
layers (where, as is our convention, the last layer does not
include a ReLU, but is simply an affine linear map).
In the construction, each activation function with $k + 1$ segments
needs $k + 1$ hidden neurons to be converted into a feed-forward architecture
(\cref{lemma:activation-function-to-relu}).
Since every KAN layer has $n^2$ activation functions,
this means we need a width of $n^2(k+1)$.
This establishes the first embedding
\begin{equation*}
    \kanspace(L, n, k)
    \subseteq
    \reluspace(L + 1, n^2(k+1)).
\end{equation*}

Conversely,
the ReLU-to-KAN conversion in \cref{thm:relu-to-kan}
keeps the number of layers constant.
The equations at the end of the proof of
\cref{thm:relu-to-kan} that define the activation functions
$\phi_{i_l}^{l, i_{l+1}}$,
show that we use activation functions that consist of at most
$2$ segments, and the width also remains unchanged.
This establishes the second embedding
\begin{equation*}
    \reluspace(L + 1, n^2(k+1))
    \subseteq
    \kanspace(L + 1, n^2(k+1), 1)
\end{equation*}
concluding the proof.

\end{proof}

\section{CONVERSION EFFICIENCY}
\label{sec:conversion-efficiency}

Let $\#:\mathcal{F} \rightarrow \mathbb{N}$ be a parameter counting function on a parameterised functional class $\mathcal{F}$. 
In the following propositions, $f, g$ represent KANs and ReLU networks. 
In contrast, $\hat{f},\hat{g}$ represent the KANs and ReLU networks that have been converted from ReLU (as in Theorem \ref{thm:relu-to-kan}) and KAN networks (as in Theorem \ref{thm:kan-to-relu}) respectively. 

\subsection{Conversion from ReLU to KAN}

\begin{lemma}
Let $g$ be a ReLU network $g \colon \mathbb{R}^n \rightarrow \mathbb{R}$ with $L \in \mathbb{N}$ hidden layers and $\mathbf{N} = [n_1, n_2, \ldots, n_L]$ neurons.
Then the network has at most 
$n\times n_1 + n_L + \sum_{i=1}^{L-1} n_i \times n_{i+1}$ parameters in the weight matrices and at most $1+\sum_{i=1}^{L} n_i$ parameters in the bias vectors, for a total of 
\[ \#(g) = 1+ n\times n_1 + 2\cdot n_L + \sum_{i=1}^{L-1} (n_i \times n_{i+1} + n_{i}).\]
\end{lemma}

\begin{proposition}
    Let $g$ be a ReLU network $g \colon \mathbb{R}^n \rightarrow \mathbb{R}$ with $L \in \mathbb{N}$ hidden layers and $\mathbf{N} = [n_1, n_2, \ldots, n_L]$ neurons.
    Let $\hat{f}$ be the KAN as constructed in the proof of Theorem \ref{thm:relu-to-kan}.
    Then,
    \[ \#(\hat f) = \#(g) + 4 \cdot (n_1+ n_2+ \ldots + n_L +1).\]
\end{proposition}

\begin{proof}
    The construction in the proof of Theorem \ref{thm:relu-to-kan} uses
    the same number of parameters as are in $g_{\theta}$ and additionally requires four parameters per application of ReLU, of which there are $n_1+ n_2+ \ldots + n_L + 1$.
\end{proof}

This translation is efficient as it scales linearly with the number of neurons in the original ReLU architecture.

\subsection{Conversion from KAN to ReLU}

\begin{proposition}
    Let $f \colon \mathbb{R}^n \rightarrow \mathbb{R}^m$ be a $L$ layer KAN network as described in Equation \ref{eq:kan}. 
    If each $\phi$ has exactly $k$ segments, the total number of parameters is given by: 
    \[ \#(f) = 2k \sum_{l = 1}^{L+1} n_l n_{l-1}. \]
    \end{proposition}
\begin{proof}
For each $\phi_{i_{l-1}}^{l, i_l}, i_l = 1,...,n_l, i_{l-1} = 1,...,n_{l-1}, l = 1,...,L+1$ there are exactly $k$ scalar parameters $a^{l,i_l, i_{l-1}}_j$ representing the slopes of the univariate linear segments and $k-1$ scalar parameters representing breakpoints  $b^{i_l}_j$ and the initial bias for a total of $2k$ scalar parameters. 
Since there are $n_l n_{l-1}$ activation functions in layer $l$, the layer has $n_l n_{l-1}\cdot 2k$ parameters, for a total of 
$\#(f) = 2k \sum_{l = 1}^{L} n_l n_{l-1}$.
\end{proof}

\begin{proposition}
    Let $f \colon \mathbb{R}^n \rightarrow \mathbb{R}$ be a $L$ layer KAN network as described in Equation \ref{eq:kan}.
    Let each $\phi$ have exactly $k$ segments.
    Let $\hat g$ be the ReLU network as constructed in the proof of Theorem \ref{thm:kan-to-relu}, then 
    \[ \#(\hat g) = 2k  \sum^{L+1}_{l=1} n_l n_{l-1}.\]
\end{proposition}

\begin{proof}
Each weight matrix $W^{(l)}$ has $n_l$ columns
and each column has length 
$\sum_{i_l=1}^{n_l}  \sum_{i_{l-1}=1}^{n_{l-1}} k_{\phi_{l, i_l, i_{l-1}}},$ and has at most 
$\sum_{i_{l-1}}^{n_{l-1}} k_{\phi_{l, i_l, i_{l-1}}}$ non-zero entries, where $k_{\phi_{l, i_l, i_{l-1}}}$ is the number of segments in the activation function $\phi_{i_{l-1}}^{l, i_l}$.

Given our assumption that every activation function has $k$ segments this means every column has at most 
$k \cdot n_{l-1}$ non-zero entries.
This means that every matrix $W^{(l)}$ has $k \cdot n_{l-1} \cdot n_l$ parameters. 

Analogously every bias vector $B^{(l)}$ has $k \cdot n_{l-1} \cdot n_l$ parameters. 

For a KAN with $L$ hidden layers and one output layer, we get
\begin{align*}
    \sum^{L+1}_{l=1} k \cdot n_{l-1} \cdot n_l + k \cdot n_{l-1} \cdot n_l =
    2k  \sum^{L+1}_{l=1} n_l n_{l-1}. 
\end{align*}
\end{proof}
This computation relies on the assumption that only non-zero values of $W^{(l)}$ are considered parameters. 
Computationally, this can be implemented with sparse matrices, in PyTorch. 
However, the width of the architecture is increased by a multiplicative factor: every layer has now $k\cdot n_l$ neurons. 

\section{POLYHEDRAL DECOMPOSITION OF KANS AND RELUS}\label{sec:polyhedral-decomposition}

In this section, we will relate the number of input regions a model differentiates between to the number of parameters needed to implement the model.

Let $\mathcal{R}: \mathcal{F} \rightarrow \mathbb{N}$ be the total number of regions in the polyhedral complex (the cardinality of $\Omega$). 
For both ReLU neworks and KANs we will consider how many parameters are needed per polytope.
In other words we will consider their representational power.

\subsection{Upper bound number of regions ReLU network}

We begin with ReLU networks, and we simply state the results of previous work.
The below proposition states that the input space of a ReLU network can be decomposed into a finite number of regions such that the network is linear in each region, and such that the network non-linearity occurs exactly on the region boundaries.

\begin{proposition}\label{prop:network-to-partition}[\citet{sudjianto2020unwrapping}]
For a ReLU network $\mathcal{N} \colon \mathbb{R}^n \rightarrow \mathbb{R}$ there is a finite partition $\Omega$ of $\mathbb{R}^n$ of cardinality $p := \# \Omega$ such that for each part $\omega \in \Omega$ there exists a piecewise linear function $f\colon \mathbb{R}^n \rightarrow \mathbb{R}^m$, and its restriction on $\omega$, denoted $f|_\omega$, is linear.
Each part is a  polytope, given by the intersection of a collection of half-spaces.
All of the half-spaces are induced by neurons. 
\end{proposition}

The below proposition gives an upper bound for the number of regions that the input space can be decomposed into.

\begin{proposition}\label{}[\cite{montufar2017}]
    Let $g$ be a ReLU network $g \colon \mathbb{R}^n \rightarrow \mathbb{R}$ with $L \in \mathbb{N}$ hidden layers and $\mathbf{N} = [n_1, n_2, \ldots, n_L]$ neurons.
    Then the number of linear regions is upper bounded by
    $\prod_{l=1}^{L} \sum_{j=0}^{d_l} {n_l\choose j},$
    where $d_l = \min\{ n, n_1, n_2, \ldots, n_l\}$,
    i.e.
    \[ \mathcal{R}(g) \leq \prod_{l=1}^{L} \sum_{j=0}^{d_l} {n_l\choose j}.\]
\end{proposition}

 \cite{serra2020empirical} find a tighter upper bound, by considering which combinations of turned off and turned on  ReLUs are possible (activation patterns).

\subsection{Upper bound number of regions of a KAN}

Now we calculate an upper bound for the number of linear regions of a KAN.

\begin{lemma}
  \label{lem:piecewise-linear}
  Let $f \colon \bR^n \to \bR^m$ be a piecewise linear map with $k$ segments,
  and let $g \colon \bR^m \to \bR^l$ be a piecewise linear map with $k'$ segments.
  Then the composition $g \circ f$ is a piecewise linear map with at most $k \cdot k'$ segments.
\end{lemma}
\begin{proof}
  Let $\omega$ be one of the segments of $f$. This means that $f|_{\omega}$ is linear.

  For every segment $\eta$ of $g$, define $\omega_\eta = \omega \cap f^{-1}(\eta)$,
  the inverse image of $\eta$ in $\omega$.
  Then $(g \circ f)|_{\omega_\eta} = g|_\eta \circ f|_{\omega_\eta}$ is the composition
  two linear functions, and hence also linear.
  So if we partition $\omega$ into at most $k'$ subsegments $\omega_\eta$,
  then $g \circ f$ is linear on all those subsegments.
  Doing this for all segments $\omega$ of $f$,
  we have found a partition of the input space of $f$ of at most $k \cdot k'$
  segments on which $g \circ f$ is linear.
\end{proof}

\begin{theorem}
    Let $f: \mathbb{R}^n \rightarrow \mathbb{R}$ be KAN network with $L$ hidden layers as described in Equation \ref{eq:kan}.
    Suppose that each activation function $\phi$ in $f$ has at most $k$ segments.
    The total number of regions $\mathcal R(f)$ of $f$ has the following upper bound:
    \begin{equation*}
      \mathcal R(f) \leq k^{n_L + \sum_{i=0}^{L-1} n_i n_{i+1}}
    \end{equation*}
    where the $n_i$ are the widths of the layers.
\end{theorem}
\begin{proof}

  We can write $f$ as a composition $f = \kanlayer_{L} \circ \mathellipsis \circ \kanlayer_0$
  of KAN layers $\kanlayer_i \colon \bR^{n_i} \to \bR^{n_{i+1}}$.
  We are going to prove that each KAN layer $\kanlayer_i$ is piecewise linear with at most $k^{n_i n_{i+1}}$ segments,
  so that the conclusion follows from \cref{lem:piecewise-linear}.

  First, consider a function
  \begin{equation*}
    \phi \colon \bR^n \to \bR, \quad x \mapsto \sum_{i=1}^n \phi_i(x_i)
  \end{equation*}
  where each $\phi_i$ is piecewise linear with at most $k$ segments.
  An example of such a function is $\kanlayer_L$.
  We will now prove that this map is piecewise linear with at most $k^n$ segments.

  For any $1 \leq i \leq n$, let $\omega_{i,1}, \mathellipsis, \omega_{i,k} \subset \bR$
  denote the $k$ segments of $\phi_i$,
  and write $\omega_{i_1, i_2, \mathellipsis, i_n} = \omega_{1,i_1} \times \cdots \times \omega_{n, i_n}$
  for the Cartesian product of one segment for every axis $i$.
  These $\omega_{i_1, i_2, \mathellipsis, i_n} = \omega_{1,i_1} \times \cdots \times \omega_{n, i_n}$
  partition the space $\bR^n$ into $k^n$ segments,
  and on every such segment, $\sum_i \phi_i$ is linear because each $\phi_i$ is linear.
  This proves that $\phi$ is piecewise linear with at most $k^n$ segments.

  Now we consider the more general function
  \begin{equation*}
    \phi \colon \bR^n \to \bR^m, \quad x \mapsto \left( \sum_{i=1}^n \phi_{i}^{j}(x_i) \right)_{j=1, \mathellipsis, m}
  \end{equation*}
  where each $\phi_{i}^{j}$ is a piecewise linear map with at most $k$ segments.
  All the $\kanlayer_\ell$ with $0 \leq \ell < L$ are of this form.
  As we will now prove, this map is piecewise linear with at most $k^{nm}$ segments.

  We write $\phi = (\phi_1, \mathellipsis, \phi_m)$,
  i.e., we write $\phi_j$ for the component maps $\phi_j \colon \bR^n \to \bR, x \mapsto \sum_{i=1}^n \phi_{i}^{j}(x_i)$.
  Our previous result shows that each $\phi_j$ is piecewise-linear with at most $k^n$ segments,
  so for each $j$ we have a partition of the input space $\mathbb R^n$
  into at most $k^n$ segments $\omega_{j,1}, \ldots, \omega_{j,k^n}$
  such that $\phi_j$ is continuous when restricted to that segment.
  By picking one such segment for each $1 \leq j \leq m$ and intersecting those chosen segments,
  we get $(k^n)^m$ subsegments of the form
  \begin{equation*}
    \omega_{j_1, \mathellipsis, j_m} := \omega_{1,j_1} \cap \cdots \cap \omega_{m,j_m}
  \end{equation*}
  that together partition the input space $\bR^n$.
  On each of these subsegments $\omega_{j_1, \mathellipsis, j_m}$,
  each $\phi_{j'}$ is linear, because $\omega_{j_1, \mathellipsis, j_m}$ is
  a subset of the segment $\omega_{j', j_{j'}}$ of $\phi_{j'}$.
  Since this is true for each $j'$, the map $\phi$ is linear on each such segment.
  This proves that $\phi$ is piecewise-linear with at most $k^{nm}$ segments,
  concluding the proof.
\end{proof}

Using this upper bound, we can approximate the number of regions that can be expressed per network parameter. 
For ReLU networks this ratio can be approximated by
\[  \dfrac{ \prod_{l=1}^{L} \sum_{j=0}^{d_l} {n_l\choose j} }{1+ n\times n_1 + 2\cdot n_L + \sum_{i=1}^{L-1} (n_i \times n_{i+1} + n_{i})}.
\]
For KANs the ratio is 
\[ \dfrac{k^{n_L + \sum_{i=0}^{L-1} n_i n_{i+1}}}{2k \sum_{l = 1}^{L+1} n_l}.
\]

This ratio is much bigger for KANs, which means that KANs are able to represent a finer polyhedral partition with fewer parameters.
While in general this may suggest that this is a more expressive class of piecewise linear functions, further research is needed to understand what functional class is represented. 

\begin{figure}
    \begin{center}
    \includegraphics[width=0.35\textwidth]{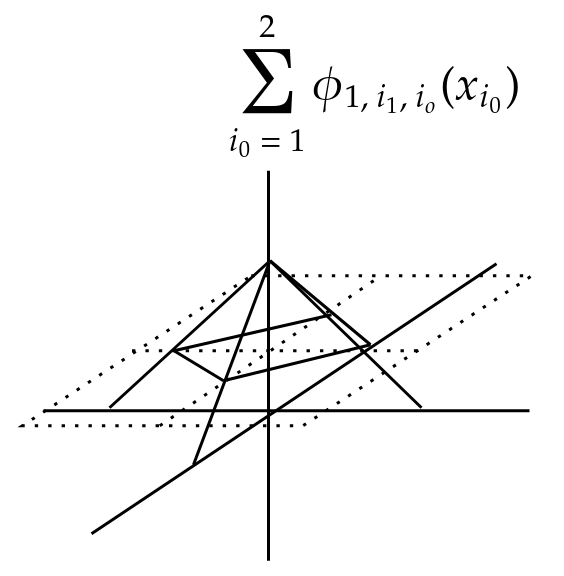}
    \end{center}
\caption{Sum of two activation functions that each have one breakpoint at the origin. A 2-dimensional hyperplane cuts through the pyramid.}
\label{fig:regions-sum}
\end{figure}

\section{CONCLUSION}

Our work develops the first analytical bridge between feedforward architectures and KANs. 
Specifically, in the context of piecewise linear functions, we are able to switch between the two representations. 
This allows users to leverage the trainability of ReLU networks and convert them to (piecewise linear) KANs in order to grab the interpretability benefits. 
In the other direction, transforming KANs into ReLU networks enables researchers and users to deploy existing techniques to analyse KANs, by importing tools from the rich literature on polyhedral decompositions \citep{huchette2023deep}, for example analysing symmetries in parameter space \citep{grigsby2023hidden} and extracting the polyhedral complex computationally \citep{montufar2014number, villani2023any, berzins2023polyhedral}. 

An important corollary of our work is that we show that any piecewise linear function can be expressed as a piecewise linear KAN.
This statement was already shown to be true for ReLU functions, i.e. any piecewise linear function can be expressed as a ReLU network \citep{he2018relu}.
Based on this fact and our transformation from ReLU networks to KANs, we can conclude the corollary.


In both directions we show the efficiency of our transformations: the transformed model in one direction (from ReLU to KANs) requires an extra linear term, and in the other direction (from KANs to ReLU) it requires no any extra non-zero parameters.

\subsection{Limitations and Future Work}

There are still a variety of open questions regarding the expressivity of KANs. 
For example,
suppose we have a piecewise linear function $\psi$,
what is the smallest (in terms of parameter count) KAN $f$ and ReLU network $g$ that can represent this function?

Polyhedral extraction methods, that compute the polyhedral partition and linear coefficients of each part, for KANs would unlock further interpretability benefits.
In particular, this would enable new insight into how parameters affect the linear parts of KANs. 


Future work could also focus on exploring methods to represent arbitrary finite piecewise linear functions with KANs and ReLU networks with a minimal number of parameters.
This would clarify the parameter efficiency of each architecture class, which could be useful for reducing storage costs and in mobile applications. 




\subsection*{Acknowledgements}

This work was supported by the EPSRC Grant EP/S023356/1 (\url{www.safeandtrustedai.org}).
We would like to thank Peter McBurney and the reviewers for their helpful comments.

%% file: appendix.tex
\onecolumn
\section{AN ANALOGOUS CONVERSION FOR KANS WITH B-SPLINE ACTIVATION FUNCTIONS}
\label{appendix:b-spline-kans}

In this appendix, we describe a conversion that applies to KANs with B-spline activation functions rather than KANs with piecewise linear activation functions, but that is very similar in other respects.
The target of this conversion is not a standard ReLU network:
feedforward networks with ReLU activations
represent piecewise linear functions,
and since KANs with B-spline activations are not necessarily
piecewise linear, this is impossible.
Instead, to carry out an analogous operation,
we need to introduce an unconventional architecture
that combines both ReLU activations (for the breakpoints in the splines) and monomial activations (for the polynomials).

Specifically, we convert a KAN with B-spline activations of degree at most $r$,
to the following architecture that we call a (ReLU, $x^r$)-architecture.
A block in this architecture consists of the following:
\begin{enumerate}
    \item an affine linear layer $\mathbb R^n \to \mathbb R^{(r+1)n'}$ for any integers $n, n' \geq 1$; followed by
    \item a ReLU layer $\bR^{(r+1)n'} \to \bR^{(r+1)n'}$,
      i.e., a component-wise application of $\relu{}$; and lastly
    \item
       the monomial activation functions $\sigma_r \colon \bR^{(r+1)n'} \to \bR^{(r+1)n'}$
       that on each of the $n'$ copies of $\bR^{r+1}$ inside
       $\bR^{(r+1)n'}$ are defined by
       $\sigma_r(y_0, y_1, \ldots, y_r) = (1, y_1^1, y_2^2, \ldots, y_r^r)$, i.e., it is the monomial $x^j$ on the $j$-th component.  
\end{enumerate}

\paragraph{Converting a (ReLU, $x^r$)-architecture to a KAN with B-splines}
This direction is simple, and is based on the same idea as for piecewise linear KANs (see \cref{thm:relu-to-kan}):
every activation function in the (ReLU, $x^r$)-architecture is in particular a spline, so you can directly replace each of the 3 types of layers in a block directly with a KAN layer.

\paragraph{Converting a KAN with B-splines to a (ReLU, $x^r$)-architecture}
This direction is more involved, but also roughly mimics the core ideas from the proof of \cref{thm:kan-to-relu}:
for each breakpoint $b_i$ in a spline, we create intermediate neurons to represent the function $x \mapsto \relu(x - b_i)$
(similar to \cref{fig:small-network}), and then post-compose these functions with monomials to create a telescoping sum
of polynomials that is equal to the original spline.
(For piecewise linear KANs, we do not use monomials, but linear functions, as illustrated in \cref{fig:small-network}.)
We do this for every layer in the KAN, and the result is a (ReLU, $x^r$)-architecture.
The rest of this appendix explains this construction in more detail.

Let's first take the example of a polynomial of degree $r$: 

$$p = \sum_{i = 0}^r a_i x^i.$$

We can write:

$$p = P_\mathbf a (x):=\mathbf a \sigma_r(\mathbf 1_{r+1}^Tx),$$
where $x\in \mathbb{R}, \mathbf{1}_{r+1} \in \mathbb{R}^{r+1}$ is a vector of ones, $\mathbf{a} = (a_0, \ldots a_r) \in \mathbb{R}^{r+1}$ is a vector of coefficients, and where $\sigma_r \colon \mathbb R^{r+1} \to \mathbb R^{r+1}$ is the map of monomial activation functions defined above,
i.e., it operates element-wise and raises the $i$-th coordinate to the $i$-th power:
$$\sigma_r(\mathbf y)_i = y_i^i \qquad \text{for } \mathbf y = (y_0, y_1, \ldots, y_r) \in \mathbb R^{r+1}.$$

This expression shows that we can use the monomial activation functions to create
a (ReLU, $x^r$)-network that represents an arbitrary polynomial.
It remains to show that we can combine this with the ReLU activation functions
to create splines.

For a polynomial $P_{\mathbf a}$ as above and a threshold $b \in \mathbb R$, we define $P_{\mathbf a}^b$ as the polynomial
that satisfies
$$P_{\mathbf a }^{b}(x) := P_{\mathbf a }(x+b).$$
With this polynomial and the $\text{ReLU}$ activation function, we can create a (ReLU, $x^r$)-network that exactly represents the function

$$P_{\mathbf a }^{b}(\text{ReLU}(x - b)).$$

This function has the useful property that on $x \geq b$, it is equal to $P_{\mathbf a}(x)$, and on $x \leq b$, it is constant and equal to $P_{\mathbf a}(b)$.

Having shown that we can represent the functions $P_{\mathbf a}^b(x)$
in a (ReLU, $x^r$)-architecture,
we can now use these functions to represent any single-valued B-spline with finitely many pieces as follows.
Given a B-spline, let $b_1, \ldots, b_{k-1}$ denote the breakpoints between the segments, and denote the $k$ polynomial functions on the polynomial segments by $P_{\mathbf a_1}(x), \ldots, P_{\mathbf a_k}(x)$. (This is similar to \cref{fig:small-network}, but with polynomials instead of linear functions.) This B-spline is exactly represented by the function

\begin{equation}
\label{eq:b-spline-to-monomial-relu}
P_{\mathbf a_1}(x) + (P_{\mathbf a_2}^{b_1} - P_{\mathbf a_1}^{b_1})(\text{ReLU}(x - b_1)) + ... + (P_{\mathbf a_k}^{b_{k-1}} - P_{\mathbf a_{k-1}}^{b_{k-1}})(\text{ReLU}(x - b_{k-1})).    
\end{equation}

This can be easily checked using the property for $P_{\mathbf a}^b$ mentioned above, from which it follows that for any $1 \leq i \leq k$ and $x \leq b_i$, all but the first $i$ terms cancel out, and the first $i$ terms form a telescoping sum that is equal to $P_{\mathbf a_1}(x) + (P_{\mathbf a_2}(x) - P_{\mathbf a_1}(x)) + \ldots + (P_{\mathbf a_i}(x) - P_{\mathbf a_{i - 1}}(x)) = P_{\mathbf a_i}(x)$. This is analogous to what we do in the proof of Lemma 1.

The function in \cref{eq:b-spline-to-monomial-relu} 
can be represented by the (ReLU, $x^r$)-architecture
because it is a sum of polynomials that can be represented by the (ReLU, $x^r$)-architecture.
We can now reason in the same way as in the main text: because we can convert the activation functions in the KAN,
we can stack and concatenate (ReLU, $x^r$)-networks to convert the entire KAN.